\documentclass{article} 
\usepackage{iclr2016_conference,times}
\usepackage{hyperref}
\usepackage{url}
\usepackage{amsmath}
\usepackage{amssymb}
\usepackage{textcomp} 
\usepackage{amsthm} 
\usepackage{booktabs} 

\usepackage{tikz}
\usetikzlibrary{matrix}

\newtheorem{theorem}{Theorem}

\title{Importance Weighted Autoencoders}

\author{
Yuri Burda, Roger Grosse \& Ruslan Salakhutdinov \\
Department of Computer Science\\
University of Toronto\\
Toronto, ON, Canada \\
\texttt{\{yburda,rgrosse,rsalakhu\}@cs.toronto.edu}
}

\newcommand{\figuresdir}[1]{{#1}}

\newcommand{\expect}{\mathbb{E}}
\newcommand{\kldiv}{\mathrm{D}_{\mathrm{KL}}}
\newcommand{\normal}{\mathcal{N}}
\newcommand{\zeroVec}{\mathbf{0}}
\newcommand{\identity}{\mathbf{I}}

\newcommand{\hid}{\mathbf{h}}
\newcommand{\hidLayer}[1]{\hid^{#1}}
\newcommand{\hidS}[1]{\hid_{#1}}
\newcommand{\params}{{\boldsymbol{\theta}}}
\newcommand{\obs}{\mathbf{x}}

\newcommand{\numLayers}{L}
\newcommand{\layerIdx}{\ell}
\newcommand{\pmfGen}{p}

\newcommand{\pmfRec}{q}

\newcommand{\varBound}{\mathcal{L}}
\newcommand{\numSamples}{k}
\newcommand{\numSamplesTwo}{m}
\newcommand{\sampleIdx}{i}
\newcommand{\sampleIdxTwo}{j}
\newcommand{\aux}{\boldsymbol{\epsilon}}
\newcommand{\auxLayer}[1]{\aux^{#1}}
\newcommand{\auxS}[1]{\aux_{#1}}

\newcommand{\isWeight}{w}
\newcommand{\isWeightS}[1]{\isWeight_{#1}}
\newcommand{\isBoundK}[1]{\mathcal{L}_{#1}}

\newcommand{\isWeightNormS}[1]{\widetilde{w_{#1}}}

\newcommand{\recMeanLayer}[1]{\boldsymbol{\mu}_{\pmfRec,#1}}
\newcommand{\recStdLayer}[1]{\boldsymbol{\sigma}_{\pmfRec,#1}}
\newcommand{\genMeanLayer}[1]{\boldsymbol{\mu}_{\pmfGen,#1}}
\newcommand{\genStdLayer}[1]{\boldsymbol{\sigma}_{\pmfGen,#1}}
\newcommand{\latentUnit}{u}

\newcommand{\normalMean}{\boldsymbol{\mu}}
\newcommand{\normalCov}{\boldsymbol{\Sigma}}

\newcommand{\indices}{I}

\begin{document}

\maketitle

\begin{abstract}
The variational autoencoder (VAE; \cite{vae}) is a recently proposed generative model pairing a top-down generative network with a bottom-up recognition network which approximates posterior inference. It typically makes strong assumptions about posterior inference, for instance that the posterior distribution is approximately factorial, and that its parameters can be approximated with nonlinear regression from the observations. As we show empirically, the VAE objective can lead to overly simplified representations which fail to use the network's entire modeling capacity. We present the importance weighted autoencoder (IWAE), a generative model with the same architecture as the VAE, but which uses a strictly tighter log-likelihood lower bound derived from importance weighting. In the IWAE, the recognition network uses multiple samples to approximate the posterior, giving it increased flexibility to model complex posteriors which do not fit the VAE modeling assumptions. We show empirically that IWAEs learn richer latent space representations than VAEs, leading to improved test log-likelihood on density estimation benchmarks. 
\end{abstract}

\section{Introduciton}

In recent years, there has been a renewed focus on learning deep generative models \citep{greedy-dbn,dbm,DARN,vae,vae2}.
A common difficulty faced by most approaches is the need to perform posterior inference during training: the log-likelihood gradients for most latent variable models are defined in terms of posterior statistics (e.g.~\citet{dbm,deep-sigmoid,DARN}). One approach for dealing with this problem is to train a recognition network alongside the generative model \citep{helmholtz_machine}. The recognition network aims to predict the posterior distribution over latent variables given the observations, and can often generate a rough approximation much more quickly than generic inference algorithms such as MCMC. 

The variational autoencoder (VAE; \citet{vae,vae2}) is a recently proposed generative model which pairs a top-down generative network with a bottom-up recognition network. 
Both networks are jointly trained to maximize a variational lower bound on the data log-likelihood. 
VAEs have recently been successful at separating style and content \citep{vae-semi-supervised,tejas_vae} and at learning to ``draw'' images in a realistic manner \citep{draw}.

VAEs make strong assumptions about the posterior distribution. Typically VAE models assume that the posterior is approximately factorial, and that its parameters can be predicted from the observables through a nonlinear regression. Because they are trained to maximize a variational lower bound on the log-likelihood, they are encouraged to learn representations where these assumptions are satisfied, i.e.~where the posterior is approximately factorial and predictable with a neural network. While this effect is beneficial, it comes at a cost: constraining the form of the posterior limits the expressive power of the model. This is especially true of the VAE objective, which harshly penalizes approximate posterior samples which are unlikely to explain the data, even if the recognition network puts much of its probability mass on good explanations.

In this paper, we introduce the importance weighted autoencoder (IWAE), a generative model which shares the VAE architecture, but which is trained with a tighter log-likelihood lower bound derived from importance weighting. The recognition network generates multiple approximate posterior samples, and their weights are averaged. As the number of samples is increased, the lower bound approaches the true log-likelihood. The use of multiple samples gives the IWAE additional flexibility to learn generative models whose posterior distributions do not fit the VAE modeling assumptions. This approach is related to reweighted wake sleep \citep{RWS}, but the IWAE is trained using a single unified objective. 
Compared with the VAE, our IWAE is able to learn richer representations with more latent dimensions, which translates into significantly higher log-likelihoods on density estimation benchmarks.

\section{Background}
\label{sec:background}

In this section, we review the variational autoencoder (VAE) model of \citet{vae}. In particular, we describe a generalization of the architecture to multiple stochastic hidden layers. We note, however, that \citet{vae} used a single stochastic hidden layer, and there are other sensible generalizations to multiple layers, such as the one presented by \citet{vae2}.

The VAE defines a generative process in terms of ancestral sampling through a cascade of hidden layers:
\begin{equation}
\pmfGen(\obs|\params) = \sum_{\hidLayer{1},\ldots,\hidLayer{\numLayers}} \pmfGen(\hidLayer{\numLayers}|\params) \pmfGen(\hidLayer{\numLayers-1}|\hidLayer{\numLayers},\params)\cdots \pmfGen(\obs|\hidLayer{1},\params). 
\end{equation}
Here, $\params$ is a vector of parameters of the variational autoencoder, and $\hid = \{\hidLayer{1}, \ldots, \hidLayer{\numLayers}\}$ denotes the stochastic hidden units, or latent variables. The dependence on $\params$ is often suppressed for clarity. For convenience, we define $\hidLayer{0} = \obs$. Each of the terms $\pmfGen(\hidLayer{\layerIdx} | \hidLayer{\layerIdx+1})$ may denote a complicated nonlinear relationship, for instance one computed by a multilayer neural network. However, it is assumed that sampling and probability evaluation are tractable for each $\pmfGen(\hidLayer{\layerIdx} | \hidLayer{\layerIdx+1})$. Note that $\numLayers$ denotes the number of \emph{stochastic} hidden layers; the deterministic layers are not shown explicitly here.
We assume the recognition model $\pmfRec(\hid|\obs)$ is defined in terms of an analogous factorization: 
\begin{equation}
\pmfRec(\hid|\obs)=q(\hidLayer{1}|\obs)q(\hidLayer{2}|\hidLayer{1})\cdots \pmfRec(\hidLayer{\numLayers}|\hidLayer{\numLayers-1}),
\end{equation}
where sampling and probability evaluation are tractable for each of the terms in the product.

In this work, we assume the same families of conditional probability distributions as \citet{vae}. In particular, the prior $\pmfGen(\hidLayer{\numLayers})$ is fixed to be a zero-mean, unit-variance Gaussian. In general, each of the conditional distributions $\pmfGen(\hidLayer{\layerIdx} |\ \hidLayer{\layerIdx+1})$ and $\pmfRec(\hidLayer{\layerIdx} | \hidLayer{\layerIdx-1})$ is a Gaussian with diagonal covariance, where the mean and covariance parameters are computed by a deterministic feed-forward neural network. For real-valued observations, $\pmfGen(\obs | \hidLayer{1})$ is also defined to be such a Gaussian; for binary observations, it is defined to be a Bernoulli distribution whose mean parameters are computed by a neural network.

The VAE is trained to maximize a variational lower bound on the log-likelihood, as derived from Jensen's Inequality:
\begin{equation}
\log \pmfGen(\obs) = \log \expect_{\pmfRec(\hid|\obs)}\left[\frac{\pmfGen(\obs,\hid)}{\pmfRec(\hid|\obs)}\right]\geq \expect_{\pmfRec(\hid|\obs)}\left[\log\frac{\pmfGen(\obs,\hid)}{\pmfRec(\hid|\obs)}\right] = \varBound(\obs). \label{eqn:vae_bound}
\end{equation}
Since $\varBound(\obs) = \log \pmfGen(\obs) - \kldiv(\pmfRec(\hid|\obs)||\pmfGen(\hid|\obs)),$ the training procedure is forced to trade off the data log-likelihood $\log \pmfGen(\obs)$ and the KL divergence from the true posterior. This is beneficial, in that it encourages the model to learn a representation where posterior inference is easy to approximate.

If one computes the log-likelihood gradient for the recognition network directly from Eqn.~\ref{eqn:vae_bound}, the result is a REINFORCE-like update rule which trains slowly because it does not use the log-likelihood gradients with respect to latent variables \citep{helmholtz_machine,nvil}. Instead, \citet{vae} proposed a reparameterization of the recognition distribution in terms of auxiliary variables with fixed distributions, such that the samples from the recognition model are a deterministic function of the inputs and auxiliary variables. While they presented the reparameterization trick for a variety of distributions, for convenience we discuss the special case of Gaussians, since that is all we require in this work. (The general reparameterization trick can be used with our IWAE as well.)

In this paper, the recognition distribution $\pmfRec(\hidLayer{\layerIdx} | \hidLayer{\layerIdx-1}, \params)$ always takes the form of a Gaussian $\normal(\hidLayer{\layerIdx} | \normalMean(\hidLayer{\layerIdx-1}, \params), \normalCov(\hidLayer{\layerIdx-1}, \params))$, whose mean and covariance are computed from the 
the states of the hidden units at the previous layer and the model parameters. This can be alternatively expressed by first sampling an auxiliary variable $\auxLayer{\layerIdx} \sim \normal(\zeroVec, \identity)$, and then applying the deterministic mapping 
\begin{equation}
\hidLayer{\layerIdx}(\auxLayer{\layerIdx}, \hidLayer{\layerIdx-1}, \params) = \normalCov(\hidLayer{\layerIdx-1}, \params)^{1/2} \auxLayer{\layerIdx} + \normalMean(\hidLayer{\layerIdx-1}, \params). \label{eqn:layer-mapping}
\end{equation}
The joint recognition distribution $\pmfRec(\hid | \obs, \params)$ over all latent variables can be expressed in terms of a deterministic mapping $\hid(\aux, \obs, \params)$, with $\aux = (\auxLayer{1}, \ldots, \auxLayer{\numLayers})$, by applying Eqn.~\ref{eqn:layer-mapping} for each layer in sequence.
Since the distribution of $\aux$ does not depend on $\params$, we can reformulate the gradient of the bound $\varBound(\obs)$ from Eqn.~\ref{eqn:vae_bound} by pushing the gradient operator inside the expectation:
\begin{align}
\nabla_{\params} \log \expect_{\hid\sim \pmfRec(\hid|\obs,\params)} \left[ \frac{\pmfGen(\obs,\hid|\params)}{\pmfRec(\hid|\obs,\params)} \right] 
&= \nabla_{\params} \expect_{\auxLayer{1}, \ldots, \auxLayer{\numLayers}\sim \normal(\zeroVec, \identity)} \left[ \log \frac{\pmfGen(\obs,\hid(\aux, \obs, \params)|\params)}{\pmfRec(\hid(\aux, \obs, \params)|\obs,\params)} \right] \\
&= \expect_{\auxLayer{1}, \ldots, \auxLayer{\numLayers} \sim \normal(\zeroVec, \identity)} \left[ \nabla_{\params}\log \frac{\pmfGen(\obs,\hid(\aux, \obs, \params)|\params)}{\pmfRec(\hid(\aux, \obs, \params)|\obs,\params)} \right]. \label{eqn:vae_reparameterization}
\end{align}
Assuming the mapping $\hid$ is represented as a deterministic feed-forward neural network, for a fixed $\aux$, the gradient inside the expectation can be computed using standard backpropagation. In practice, one approximates the expectation in Eqn.~\ref{eqn:vae_reparameterization} by generating $\numSamples$ samples of $\aux$ and applying the Monte Carlo estimator
\begin{equation}
\label{eqn:ksample_var_grad}
\frac{1}{\numSamples}\sum_{\sampleIdx=1}^\numSamples\nabla_{\params}\log \isWeight \left(\obs, \hid(\auxS{\sampleIdx}, \obs, \params), \params\right)
\end{equation}
with $\isWeight(\obs,\hid,\params)= \pmfGen(\obs,\hid|\params) / \pmfRec(\hid|\obs,\params)$. This is an unbiased estimate of $\nabla_\params \varBound(\obs)$. We note that the VAE update and the basic REINFORCE-like update are both unbiased estimators of the same gradient, but the VAE update tends to have lower variance in practice because it makes use of the log-likelihood gradients with respect to the latent variables.

\section{Importance Weighted Autoencoder}
\label{sec:iwae}

The VAE objective of Eqn.~\ref{eqn:vae_bound} heavily penalizes approximate posterior samples which fail to explain the observations. This places a strong constraint on the model, since the variational assumptions must be approximately satisfied in order to achieve a good lower bound. In particular, the posterior distribution must be approximately factorial and predictable with a feed-forward neural network. This VAE criterion may be too strict; a recognition network which places only a small fraction (e.g.~20\%) of its samples in the region of high posterior probability region may still be sufficient for performing accurate inference. If we lower our standards in this way, this may give us additional flexibility to train a generative network whose posterior distributions do not fit the VAE assumptions. This is the motivation behind our proposed algorithm, the Importance Weighted Autoencoder (IWAE).

Our IWAE uses the same architecture as the VAE, with both a generative network and a recognition network. The difference is that it is trained to maximize a different lower bound on $\log \pmfGen(\obs)$. In particular, we use the following lower bound, corresponding to the $\numSamples$-sample importance weighting estimate of the log-likelihood:
\begin{equation}
\isBoundK{\numSamples}(\obs)=\expect_{\hidS{1},\ldots,\hidS{\numSamples} \sim \pmfRec(\hid | \obs)} \left[ \log \frac{1}{\numSamples} \sum_{\sampleIdx=1}^\numSamples \frac{\pmfGen(\obs,\hidS{\sampleIdx})}{\pmfRec(\hidS{\sampleIdx}|\obs)} \right]. \label{eqn:is_bound}
\end{equation}
Here, $\hidS{1}, \ldots, \hidS{\numSamples}$ are sampled independently from the recognition model. The term inside the sum corresponds to the unnormalized importance weights for the joint distribution, which we will denote as $\isWeightS{\sampleIdx}=\pmfGen(\obs,\hidS{\sampleIdx}) / \pmfRec(\hidS{\sampleIdx}|\obs)$.

This is a lower bound on the marginal log-likelihood, as follows from Jensen's Inequality and the fact that the average importance weights are an unbiased estimator of $\pmfGen(\obs)$:
\begin{equation}
\isBoundK{\numSamples}=\expect \left[ \log \frac{1}{\numSamples} \sum_{\sampleIdx=1}^\numSamples \isWeightS{\sampleIdx} \right] \leq \log \expect \left[ \frac{1}{\numSamples} \sum_{\sampleIdx=1}^\numSamples \isWeightS{\sampleIdx} \right] =\log \pmfGen(\obs),
\end{equation}
where the expectations are with respect to $\pmfRec(\hid | \obs)$.

It is perhaps unintuitive that importance weighting would be a reasonable estimator in high dimensions. Observe, however, that the special case of $\numSamples=1$ is equivalent to the standard VAE objective shown in Eqn.~\ref{eqn:vae_bound}. Using more samples can only improve the tightness of the bound:
\begin{theorem}
For all $\numSamples$, the lower bounds satisfy
\begin{equation}
\log \pmfGen(\obs) \geq \isBoundK{\numSamples + 1} \geq \isBoundK{\numSamples}.
\end{equation}
Moreover, if $\pmfGen(\hid, \obs) / \pmfRec(\hid | \obs)$ is bounded, then $\isBoundK{\numSamples}$ approaches $\log \pmfGen(\obs)$ as $\numSamples$ goes to infinity.
\end{theorem}
\begin{proof}
See Appendix A.
\end{proof}

The bound $\isBoundK{\numSamples}$ can be estimated using the straightforward Monte Carlo estimator, where we generate samples from the recognition network and average the importance weights. One might worry about the variance of this estimator, since importance weighting famously suffers from extremely high variance in cases where the proposal and target distributions are not a good match. However, as our estimator is based on the \emph{log} of the average importance weights, it does not suffer from high variance. This argument is made more precise in Appendix B.

\subsection{Training procedure}

To train an IWAE with a stochastic gradient based optimizer, we use an unbiased estimate of the gradient of $\isBoundK{\numSamples}$, defined in Eqn.~\ref{eqn:is_bound}. As with the VAE, we use the reparameterization trick to derive a low-variance upate rule:
\begin{small}
\begin{align}
\nabla_\params \isBoundK{\numSamples}(\obs) = \nabla_\params \expect_{\hidS{1},\ldots,\hidS{\numSamples}} \left[ \log \frac{1}{\numSamples} \sum_{\sampleIdx=1}^\numSamples \isWeightS{\sampleIdx} \right] 
&= \nabla_\params \expect_{\auxS{1},\ldots,\auxS{\numSamples}} \left[ \log \frac{1}{\numSamples} \sum_{\sampleIdx=1}^\numSamples \isWeight(\obs, \hid(\obs, \auxS{\sampleIdx}, \params), \params) \right] \\
&= \expect_{\auxS{1},\ldots,\auxS{\numSamples}} \left[ \nabla_\params \log \frac{1}{\numSamples} \sum_{\sampleIdx=1}^\numSamples \isWeight(\obs, \hid(\obs, \auxS{\sampleIdx}, \params), \params) \right] \\
&= \expect_{\auxS{1},\ldots,\auxS{\numSamples}} \left[ \sum_{\sampleIdx=1}^\numSamples \isWeightNormS{\sampleIdx} \nabla_\params \log \isWeight(\obs, \hid(\obs, \auxS{\sampleIdx}, \params), \params) \right], \label{eqn:iwae_gradient}
\end{align}
\end{small}

\noindent where $\auxS{1},\ldots,\auxS{\numSamples}$ are the same auxiliary variables as defined in Section~\ref{sec:background} for the VAE, $\isWeightS{\sampleIdx}=\isWeight(\obs,\hid(\obs, \auxS{\sampleIdx}, \params),\params)$ are the importance weights expressed as a deterministic function, and $\isWeightNormS{\sampleIdx} = \isWeightS{\sampleIdx}/\sum_{\sampleIdx=1}^\numSamples \isWeightS{\sampleIdx}$ are the normalized importance weights.

In the context of a gradient-based learning algorithm, we draw $\numSamples$ samples from the recognition network (or, equivalently, $\numSamples$ sets of auxiliary variables), and use the Monte Carlo estimate of Eqn.~\ref{eqn:iwae_gradient}:
\begin{equation}
\label{eqn:grad_ksample_objective}
\sum_{\sampleIdx=1}^\numSamples \isWeightNormS{\sampleIdx} \nabla_{\params}\log \isWeight\left(\obs, \hid(\auxS{\sampleIdx}, \obs, \params), \params\right).
\end{equation}
In the special case of $\numSamples = 1$, the single normalized weight $\isWeightNormS{1}$ takes the value 1, and one obtains the VAE update rule.

We unpack this update because it does not quite parallel that of the standard VAE.\footnote{\citet{vae} separated out the KL divergence in the bound of Eqn.~\ref{eqn:vae_bound} in order to achieve a simpler and lower-variance update. Unfortunately, no analogous trick applies for $\numSamples > 1$. In principle, the IWAE updates may be higher variance for this reason. However, in our experiments, we observed that the performance of the two update rules was indistinguishable in the case of $\numSamples=1$.} The gradient of the log weights decomposes as:
\begin{equation}
\nabla_\params \log \isWeight(\obs, \hid(\obs, \auxS{\sampleIdx}, \params), \params) = \nabla_\params \log \pmfGen(\obs,\hid(\obs, \auxS{\sampleIdx}, \params)|\params) - \nabla_\params \log \pmfRec(\hid(\obs, \auxS{\sampleIdx}, \params)|\obs,\params).
\end{equation}
The first term encourages the generative model to assign high probability to each $\hidLayer{\layerIdx}$ given $\hidLayer{\layerIdx+1}$ (following the convention that $\obs = \hidLayer{0}$). It also encourages the recognition network to adjust the hidden representations so that the generative network makes better predictions. In the case of a single stochastic layer (i.e.~$\numLayers=1$), the combination of these two effects is equivalent to backpropagation in a stochastic autoencoder. The second term of this update encourages the recognition network to have a spread-out distribution over predictions. This update is averaged over the samples with weight proportional to the importance weights, motivating the name ``importance weighted autoencoder.'' 

The dominant computational cost in IWAE training is computing the activations and parameter gradients needed for $\nabla_\params \log \isWeight(\obs, \hid(\obs, \auxS{\sampleIdx}, \params), \params)$. This corresponds to the forward and backward passes in backpropagation. In the basic IWAE implementation, both passes must be done independently for each of the $\numSamples$ samples. Therefore, the number of operations scales linearly with $\numSamples$. In our GPU-based implementation, the samples are processed in parallel by replicating each training example $\numSamples$ times within a mini-batch. 

One can greatly reduce the computational cost by adding another form of stochasticity. Specifically, only the forward pass is needed to compute the importance weights. The sum in Eqn.~\ref{eqn:grad_ksample_objective} can be stochastically approximated by choosing a single sample $\auxS{\sampleIdx}$ proprtional to its normalized weight $\isWeightNormS{\sampleIdx}$ and then computing $\nabla_\params \log \isWeight(\obs, \hid(\obs, \auxS{\sampleIdx}, \params), \params)$. This method requires $\numSamples$ forward passes and one backward pass per training example. Since the backward pass requires roughly twice as many add-multiply operations as the forward pass, for large $\numSamples$, this trick reduces the number of add-multiply operations by roughly a factor of 3. This comes at the cost of increased variance in the updates, but empirically we have found the tradeoff to be favorable.

\section{Related work}
\label{sec:related-work}

There are several broad families of approaches to training deep generative models. Some models are defined in terms of Boltzmann distributions \citep{rbm,dbm}. This has the advantage that many of the conditional distributions are tractable, but the inability to sample from the model or compute the partition function has been a major roadblock \citep{ais-rbm}. Other models are defined in terms of belief networks \citep{deep-sigmoid,DARN}. These models are tractable to sample from, but the conditional distributions become tangled due to the explaining away effect. 

One strategy for dealing with intractable posterior inference is to train a recognition network which approximates the posterior. A classic approach was the wake-sleep algorithm, used to train Helmholtz machines \citep{helmholtz_machine}. The generative model was trained to model the conditionals inferred by the recognition net, and the recognition net was trained to explain synthetic data generated by the generative net. Unfortunately, wake-sleep trained the two networks on different objective functions. Deep autoregressive networks \citep{DARN} consisted of deep generative and recognition networks trained using a single variational lower bound. Neural variational inference and learning \citep{nvil} is another algorithm for training recognition networks which reduces stochasticity in the updates by training a third network to predict reward baselines in the context of the REINFORCE algorithm \citep{reinforce}. \citet{dbm-recognition} used a recognition network to approximate the posterior distribution in deep Boltzmann machines.

Variational autoencoders \citep{vae,vae2}, as described in detail in Section~\ref{sec:background}, are another combination of generative and recognition networks, trained with the same variational objective as DARN and NVIL. However, in place of REINFORCE, they reduce the variance of the updates through a clever reparameterization of the random choices. The reparameterization trick is also known as ``backprop through a random number generator'' \citep{reinforce}.

One factor distinguishing VAEs from the other models described above is that the model is described in terms of a simple distribution followed by a deterministic mapping, rather than a sequence of stochastic choices. Similar architectures have been proposed which use different training objectives. Generative adversarial networks \citep{generative_adversarial} train a generative network and a recognition network which act in opposition: the recognition network attempts to distinguish between training examples and generated samples, and the generative model tries to generate samples which fool the recognition network. Maximum mean discrepancy (MMD) networks \citep{yujia_mmd,karolina_mmd} attempt to generate samples which match a certain set of statistics of the training data. They can be viewed as a kind of adversarial net where the adversary simply looks at the set of pre-chosen statistics \citep{karolina_mmd}. In contrast to VAEs, the training criteria for adversarial nets and MMD nets are not based on the data log-likelihood.

Other researchers have derived log-probability lower bounds by way of importance sampling. \citet{charlie-stochastic} and \citet{jimmy-attention} avoided recognition networks entirely, instead performing inference using importance sampling from the prior. \citet{dechter-importance} presented a variety of graphical model inference algorithms based on importance weighting. Reweighted wake-sleep (RWS) of \cite{RWS} is another recognition network approach which combines the original wake-sleep algorithm with updates to the generative network equivalent to gradient ascent on our bound $\isBoundK{\numSamples}$. However, \citet{RWS} interpret this update as following a biased estimate of $\nabla_\params \log \pmfGen(\obs)$, whereas we interpret it as following an unbiased estimate of $\nabla_\params \isBoundK{\numSamples}$. The IWAE also differs from RWS in that the generative and recognition networks are trained to maximize a single objective, $\isBoundK{\numSamples}$. By contrast, the $q$-wake and sleep steps of RWS do not appear to be related to $\isBoundK{\numSamples}$. Finally, the IWAE differs from RWS in that it makes use of the reparameterization trick.

Apart from our approach of using multiple approximate posterior samples, another way to improve the flexibility of posterior inference is to use a more sophisticated algorithm than importance sampling. Examples of this approach include normalizing flows \citep{normalizing_flows} and the Hamiltonian variational approximation of \citet{bridging_the_gap}. 

After the publication of this paper the authors learned that the idea of using an importance weighted lower bound for training variational autoencoders has been independently explored by Laurent Dinh and Vincent Dumoulin, and preliminary results of their work were presented at the 2014 CIFAR NCAP Deep Learning summer school.

\section{Experimental results}
\label{sec:experiments}

We have compared the generative performance of the VAE and IWAE in terms of their held-out log-likelihoods on two density estimation benchmark datasets. We have further investigated a particular issue we have observed with VAEs and IWAEs, namely that they learn latent spaces of significantly lower dimensionality than the modeling capacity they are allowed. We tested whether the IWAE training method ameliorates this effect.

\subsection{Evaluation on density estimation}

We evaluated the models on two benchmark datasets: MNIST, a dataset of images of handwritten digits \citep{mnist}, and Omniglot, a dataset of handwritten characters in a variety of world alphabets \citep{omniglot}. In both cases, the observations were binarized $28 \times 28$ images.\footnote{Unfortunately, the generative modeling literature is inconsistent about the method of binarization, and different choices can lead to considerably different log-likelihood values. We follow the procedure of \citet{ais-rbm}: the binary-valued observations are sampled with expectations equal to the real values in the training set. See Appendix D for an alternative binarization scheme.} We used the standard splits of MNIST into 60,000 training and 10,000 test examples, and of Omniglot into 24,345 training and 8,070 test examples.

We trained models with two architectures:
\begin{enumerate}
\item An architecture with a single stochastic layer $\hidLayer{1}$ with 50 units. In between the observations and the stochastic layer were two deterministic layers, each with 200 units.
\item An architecture with two stochastic layers $\hidLayer{1}$ and $\hidLayer{2}$, with 100 and 50 units, respectively. In between $\obs$ and $\hidLayer{1}$ were two deterministic layers with 200 units each. In between $\hidLayer{1}$ and $\hidLayer{2}$ were two deterministic layers with 100 units each.
\end{enumerate}
All deterministic hidden units used the $\tanh$ nonlinearity. All stochastic layers used Gaussian distributions with diagonal covariance, with the exception of the visible layer, which used Bernoulli distributions. An $\exp$ nonlinearity was applied to the predicted variances of the Gaussian distributions. The network architectures are summarized in Appendix C.

All models were initialized with the heuristic of \citet{initialization}. For optimization, we used Adam \citep{adam} with parameters $\beta_1=0.9,\beta_2=0.999,\epsilon=10^{-4}$ and minibaches of size $20$. The training proceeded for $3^i$ passes over the data with learning rate of $0.001\cdot 10^{-i/7}$ for $i=0\ldots 7$ (for a total of $\sum_{i=0}^7 3^i=3280$ passes over the data). This learning rate schedule was chosen based on preliminary experiments training a VAE with one stochastic layer on MNIST. 

For each number of samples $\numSamples\in \{1,5,50\}$ we trained a VAE with the gradient of $\varBound(\bf{x})$ estimted as in Eqn.~\ref{eqn:ksample_var_grad} and an IWAE with the gradient estimated as in Eqn.~\ref{eqn:grad_ksample_objective}. For each $\numSamples$, the VAE and the IWAE were trained for approximately the same length of time.

All log-likelihood values were estimated as the mean of $\isBoundK{5000}$ on the test set. Hence, the reported values are stochastic lower bounds on the true value, but are likely to be more accurate than the lower bounds used for training. 

The log-likelihood results are reported in Table~\ref{tbl:results}. Our VAE results are comparable to those previously reported in the literature. We observe that training a VAE with $\numSamples > 1$ helped only slightly. By contrast, using multiple samples improved the IWAE results considerably on both datasets. Note that the two algorithms are identical for $\numSamples=1$, so the results ought to match up to random variability. 

On MNIST, IWAE with two stochastic layers and $\numSamples=50$ achieves a log-likelihood of -82.90 on the permutation-invariant model on this dataset. By comparison, deep belief networks achieved log-likelihood of approximately -84.55 nats \citep{dbn_chib}, and deep autoregressive networks achieved log-likelihood of -84.13 nats \citep{DARN}. \citet{draw}, who exploited spatial structure, achieved a log-likelihood of -80.97. We did not find overfitting to be a serious issue for either the VAE or the IWAE: in both cases, the training log-likelihood was 0.62 to 0.79 nats higher than the test log-likelihood. We present samples from our models in Appendix E.

For the OMNIGLOT dataset, the best performing IWAE has log-likelihood of -103.38 nats, which is slightly worse than the log-likelihood of -100.46 nats achieved by a Restricted Boltzmann Machine with 500 hidden units trained with persistent contrastive divergence \citep{raise}. RBMs trained with centering or FANG methods achieve a similar performance of around -100 nats \citep{fang}. The training log-likelihood for the models we trained was 2.39 to 2.65 nats higher than the test log-likelihood.

\begin{table}
\vspace{-0.3in}
\begin{center}
\begin{tabular}{@{}ll cc cc cc cc cc cc@{}}
              &     & \multicolumn{4}{c}{MNIST} & \multicolumn{4}{c}{OMNIGLOT} \\
\cmidrule(r){3-6} \cmidrule(l){7-10} 
              &     & \multicolumn{2}{c}{VAE} & \multicolumn{2}{c}{IWAE} & \multicolumn{2}{c}{VAE} & \multicolumn{2}{c}{IWAE} \\
\cmidrule(r){3-4} \cmidrule(l){5-6} \cmidrule(lr){7-8} \cmidrule(l){9-10} 
\shortstack{\# stoch. \\ layers} & $\numSamples$ & NLL               & \shortstack{active \\ units}             & NLL               & \shortstack{active \\ units}  & NLL                & \shortstack{active \\ units}   & NLL               & \shortstack{active \\ units}  \\ 
\cmidrule(lr){1-1} \cmidrule(lr){2-2} \cmidrule(lr){3-3}\cmidrule(lr){4-4}\cmidrule(lr){5-5}\cmidrule(lr){6-6} \cmidrule(lr){7-7}\cmidrule(lr){8-8}\cmidrule(lr){9-9}\cmidrule(lr){10-10}\cmidrule(lr){11-11}      \midrule
1             & 1   & 86.76 &  19  & 86.76&  19  & 108.11 & 28 & 108.11 & 28 \\
              & 5   & 86.47 &  20  & 85.54&  22  & 107.62 & 28 & 106.12 & 34 \\
              & 50  & 86.35 &  20  & 84.78&  25  & 107.80 & 28 & 104.67 & 41 \\ \midrule
2             & 1   & 85.33 &  16+5 & 85.33&  16+5& 107.58 &28+4& 107.56 & 30+5 \\
              & 5   & 85.01 &  17+5 & 83.89&  21+5& 106.31 &30+5& 104.79 & 38+6 \\
              & 50  & 84.78 &  17+5 & 82.90&  26+7& 106.30 &30+5& 103.38 & 44+7 \\ \bottomrule
\end{tabular}
\end{center}
\vspace{-0.1in}
\caption{\small Results on density estimation and the number of active latent dimensions. For models with two latent layers, ``$k_1+k_2$'' denotes $k_1$ active units in the first layer and $k_2$ in the second layer. The generative performance of IWAEs improved with increasing $\numSamples$, while that of VAEs benefitted only slightly. Two-layer models achieved better generative performance than one-layer models. }
\label{tbl:results}
\end{table}

\subsection{Latent space representation}
\label{sec:latent}

We have observed that both VAEs and IWAEs tend to learn latent representations with effective dimensions far below their capacity. Our next set of experiments aimed to quantify this effect and determine whether the IWAE objective ameliorates this effect. 




If a latent dimension encodes useful information about the data, we would expect its distribution to change depending on the observations. Based on this intuition, we measured activity of a latent dimension $\latentUnit$ using the statistic $A_\latentUnit = \mathrm{Cov}_{\obs} \left( \expect_{\latentUnit \sim \pmfRec(\latentUnit | \obs)} [\latentUnit] \right)$. We defined the dimension $\latentUnit$ to be active if $A_\latentUnit > 10^{-2}$. We have observed two pieces of evidence that this criterion is both well-defined and meaningful:
\begin{enumerate}
\item The distribution of $A_\latentUnit$ for a trained model consisted of two widely separated modes, as shown in Appendix C.
\item To confirm that the inactive dimensions were indeed insignificant to the predictions, we evaluated all models with the inactive dimensions removed. In all cases, this changed the test log-likelihood by less than $0.06$ nats.
\end{enumerate}

In Table \ref{tbl:results}, we report the numbers of active units for all conditions. In all conditions, the number of active dimensions was far less than the total number of dimensions. Adding more latent dimensions did not increase the number of active dimensions. Interestingly, in the two-layer models, the second layer used very little of its modeling capacity: the number of active dimensions was always less than~10. In all cases with $\numSamples > 1$, the IWAE learned more latent dimensions than the VAE. Since this coincided with higher log-likelihood values, we speculate that a larger number of active dimensions reflects a richer latent representation.

Superficially, the phenomenon of inactive dimensions appears similar to the problem of ``units dying out'' in neural networks and latent variable models, an effect which is often ascribed to difficulties in optimization. For example, if a unit is inactive, it may never receive a meaningful gradient signal because of a plateau in the optimization landscape. In such cases, the problem may be avoided through a better initialization. To determine whether the inactive units resulted from an optimization issue or a modeling issue, we took the best-performing VAE and IWAE models from Table~\ref{tbl:results}, and continued training the VAE model using the IWAE objective and vice versa. In both cases, the model was trained for an additional $3^7$ passes over the data with a learning rate of $10^{-4}$.

The results are shown in Table~\ref{tbl:optima}. We found that continuing to train the VAE with the IWAE objective increased the number of active dimensions and the test log-likelihood, while continuing to train the IWAE with the VAE objective did the opposite. The fact that training with the VAE objective actively reduces both the number of active dimensions and the log-likelihood strongly suggests that inactivation of the latent dimensions is driven by the objective functions rather than by optimization issues. On the other hand, optimization also appears to play a role, as the results in Table~\ref{tbl:optima} are not quite identical to those in Table~\ref{tbl:results}.

\begin{table}
\vspace{-0.3in}
\begin{center}
\begin{tabular}{@{}c ccc ccc@{}}
& \multicolumn{3}{c}{First stage} & \multicolumn{3}{c}{Second stage} \\
\cmidrule(r){2-4} \cmidrule(l){5-7} 
& trained as   & NLL   & active units & trained as             & NLL   & active units \\ 
\cmidrule(lr){2-2} \cmidrule(lr){3-3} \cmidrule(lr){4-4} \cmidrule(lr){5-5} \cmidrule(lr){6-6} \cmidrule(lr){7-7} \midrule
Experiment 1 & VAE          & 86.76 & 19           & IWAE, $\numSamples=50$           & 84.88 & 22           \\ \midrule
Experiment 2 & IWAE, $\numSamples=50$ & 84.78 & 25           & VAE                    & 86.02 & 23           \\
\bottomrule
\end{tabular}
\end{center}
\vspace{-0.1in}
\caption{\small Results of continuing to train a VAE model with the IWAE objective, and vice versa. Training the VAE with the IWAE objective increased the latent dimension and test log-likelihood, while training the IWAE with the VAE objective had the opposite effect.}
\label{tbl:optima}
\end{table}

\section{Conclusion}
\label{sec:conclusion}

In this paper, we presented the importance weighted autoencoder, a variant on the VAE trained by maximizing a tighter log-likelihood lower bound derived from importance weighting. We showed empirically that IWAEs learn richer latent representations and achieve better generative performance than VAEs with equivalent architectures and training time. We believe this method may improve the flexibility of other generative models currently trained with the VAE objective.

\section{Acknowledgements}
This research was supported by NSERC, the Fields Institute, 
and Samsung. 

\begin{small}
\bibliography{iwae}
\bibliographystyle{iwae}
\end{small}

\section*{Appendix A}
\label{apx:proofs}


{\bf Proof of Theorem 1.} We need to show the following facts about the log-likelihood lower bound $\isBoundK{\numSamples}$:
 \begin{enumerate} 
 \item $\log \pmfGen(\obs) \geq \isBoundK{\numSamples}$,
 \item $\isBoundK{\numSamples}\geq \isBoundK{\numSamplesTwo}$ for $\numSamples\geq \numSamplesTwo$,
 \item $\log \pmfGen(\obs) = \lim_{\numSamples\to \infty} \isBoundK{\numSamples}$, assuming $\pmfGen(\hid, \obs) / \pmfRec(\hid | \obs)$ is bounded.
 \end{enumerate}
We prove each in turn:
\begin{enumerate}
\item
It follows from Jensen's inequality that
\begin{equation}
\isBoundK{\numSamples}=\expect\left[\log \frac{1}{\numSamples} \sum_{\sampleIdx=1}^\numSamples \frac{\pmfGen(\obs,\hidS{\sampleIdx})}{\pmfRec(\hidS{\sampleIdx}|\obs)}  \right]\leq \log \expect\left[\frac{1}{\numSamples} \sum_{\sampleIdx=1}^\numSamples \frac{\pmfGen(\obs,\hidS{\sampleIdx})}{\pmfRec(\hidS{\sampleIdx}|\obs)}  \right]=\log \pmfGen(\obs)
\end{equation}
\item Let $\indices\subset\{1,\ldots,\numSamples\}$ with $|\indices|=\numSamplesTwo$ be a uniformly distributed subset of distinct indices from $\{1,\ldots,\numSamples\}$. We will use the following simple observation: $\expect_{\indices=\{\sampleIdx_1,\ldots,\sampleIdx_\numSamplesTwo\}} \left[ \frac{a_{\sampleIdx_1}+\ldots+a_{\sampleIdx_\numSamplesTwo}}{\numSamplesTwo} \right] =\frac{a_1+\ldots+a_\numSamples}{\numSamples}$ for any sequence of numbers $a_1,\ldots,a_\numSamples$.

Using this observation and Jensen's inequality, we get
\begin{align}
\isBoundK{\numSamples} &= \expect_{\hidS{1},\ldots,\hidS{\numSamples}} \left[ \log \frac{1}{\numSamples} \sum_{\sampleIdx=1}^\numSamples \frac{\pmfGen(\obs,\hidS{\sampleIdx})}{\pmfRec(\hidS{\sampleIdx}|\obs)} \right]\\
&= \expect_{\hidS{1},\ldots,\hidS{\numSamples}} \left[ \log\expect_{\indices=\{\sampleIdx_1,\ldots,\sampleIdx_\numSamplesTwo\}}\left[\frac{1}{\numSamplesTwo} \sum_{\sampleIdxTwo=1}^\numSamplesTwo \frac{\pmfGen(\obs,\hidS{\sampleIdx_\sampleIdxTwo})}{\pmfRec(\hidS{\sampleIdx_\sampleIdxTwo}|\obs)} \right] \right] \\
& \geq \expect_{\hidS{1},\ldots,\hidS{\numSamples}} \left[ \expect_{\indices=\{\sampleIdx_1,\ldots,\sampleIdx_\numSamplesTwo\}} \left[ \log \frac{1}{\numSamplesTwo} \sum_{\sampleIdxTwo=1}^\numSamplesTwo \frac{\pmfGen(\obs,\hidS{\sampleIdx_\sampleIdxTwo})}{\pmfRec(\hidS{\sampleIdx_\sampleIdxTwo}|\obs)} \right] \right] \\
&=\expect_{\hidS{1},\ldots,\hidS{\numSamplesTwo}}\left[\log \frac{1}{\numSamplesTwo} \sum_{\sampleIdx=1}^\numSamplesTwo \frac{\pmfGen(\obs,\hidS{\sampleIdx})}{\pmfRec(\hidS{\sampleIdx}|\obs)} \right]=\isBoundK{\numSamplesTwo}
\end{align}
\item Consider the random variable 
$M_k = \frac{1}{\numSamples} \sum_{\sampleIdx=1}^\numSamples \frac{\pmfGen(\obs,\hidS{\sampleIdx})}{\pmfRec(\hidS{\sampleIdx}|\obs)} $. If $\pmfGen(\hid, \obs) / \pmfRec(\hid | \obs)$ is bounded, then it follows from the strong law of large numbers that $M_k$ converges to $\expect_{\pmfRec(\hidS{\sampleIdx}|\obs)} \left[ \frac{\pmfGen(\obs,\hidS{\sampleIdx})}{\pmfRec(\hidS{\sampleIdx}|\obs)}  \right] = \pmfGen(\obs)$ almost surely. Hence $\isBoundK{\numSamples}=\expect \log [M_k]$ converges to $\log \pmfGen(\obs)$ as $\numSamples\to \infty$.
\end{enumerate}

\section*{Appendix B}
\label{app:is_variance}

It is well known that the variance of an unnormalized importance sampling based estimator can be extremely large, or even infinite, if the proposal distribution is not well matched to the target distribution. Here we argue that the Monte Carlo estimator of $\isBoundK{\numSamples}$, described in Section~\ref{sec:iwae}, does not suffer from large variance. More precisely, we bound the mean absolute deviation (MAD). While this does not directly bound the variance, it would be surprising if an estimator had small MAD yet extremely large variance.

Suppose we have a strictly positive unbiased estimator $\hat{Z}$ of a positive quantity $Z$, and we wish to use $\log \hat{Z}$ as an estimator of $\log Z$. By Jensen's inequality, this is a biased estimator, i.e.~$\expect[\log \hat{Z}] \leq \log Z$. Denote the bias as $\delta = \log Z - \expect[\log \hat{Z}]$. We start with the observation that $\log \hat{Z}$ is unlikely to overestimate $\log Z$ by very much, as can be shown with Markov's Inequality:
\begin{equation}
{\rm Pr}(\log \hat{Z} > \log Z + b) \leq e^{-b}. \label{eqn:markov-bound}
\end{equation}

Let $(X)_+$ denote $\max(X, 0)$. We now use the above facts to bound the MAD:
\begin{align}
\expect \left[ \left|\log \hat{Z} - \expect[\log \hat{Z}] \right| \right]
&= 2 \expect \left[ \left(\log \hat{Z} - \expect[\log \hat{Z}] \right)_+ \right] \label{eqn:mad-step} \\
&= 2 \expect \left[ \left(\log \hat{Z} - \log Z + \log Z - \expect[\log \hat{Z}] \right)_+ \right] \\
&\leq 2 \expect \left[ \left(\log \hat{Z} - \log Z \right)_+ + \left( \log Z - \expect[\log \hat{Z}] \right)_+ \right] \\
&= 2 \expect \left[ \left(\log \hat{Z} - \log Z \right)_+ \right] + 2 \delta \\
&= 2 \int_0^\infty {\rm Pr} \left( \log \hat{Z} - \log Z > t \right) \mathrm{d}t + 2 \delta \label{eqn:expectation-step} \\
&\leq 2 \int_0^\infty e^{-t} \mathrm{d}t + 2 \delta \label{eqn:bound-step} \\
&= 2 + 2 \delta
\end{align}
Here, (\ref{eqn:mad-step}) is a general formula for the MAD, (\ref{eqn:expectation-step}) uses the formula $\expect[Y] = \int_0^\infty \mathrm{Pr}(Y > t)\, \mathrm{d}t$ for a nonnegative random variable $Y$, and (\ref{eqn:bound-step}) applies the bound (\ref{eqn:markov-bound}). Hence, the MAD is bounded by $2 + 2 \delta$. In the context of IWAE, $\delta$ corresponds to the gap between $\isBoundK{\numSamples}$ and $\log \pmfGen(\obs)$.

\section*{Appendix C}
\label{app:visualizations}

\subsection*{Network architectures}
Here is a summary of the network architectures used in the experiments:

 \begin{tikzpicture}
 \node at (0, 0.3) {$\pmfRec(\hidLayer{1}|\obs)=\normal(\hidLayer{1}|\recMeanLayer{1},\mathrm{diag}(\recStdLayer{1}))$} ;
  \matrix (m) [matrix of math nodes,row sep=0em,column sep=4em,minimum width=2em,nodes={draw}] at (7, 0)
  {
     \obs & \mbox{200d} & \mbox{200d} & \recMeanLayer{1} \\
       &                  &                  & \recStdLayer{1} \\};
  \path[-stealth]
    (m-1-1) edge node [above] {lin+$\tanh$} (m-1-2) 
    (m-1-2) edge node [above] {lin+$\tanh$} (m-1-3)
    (m-1-3) edge node [above] {lin}         (m-1-4)
	        edge node [below] {lin+$\exp$}  (m-2-4);
\end{tikzpicture} 
\begin{tikzpicture}
 \node at (0, 0.3) {$\pmfRec(\hidLayer{2}|\hidLayer{1})=\normal(\hidLayer{2}|\recMeanLayer{2},\mathrm{diag}(\recStdLayer{2}))$} ;
  \matrix (m) [matrix of math nodes,row sep=0em,column sep=4em,minimum width=2em,nodes={draw}] at (7, 0)
  {
     \hidLayer{1} & \mbox{100d} & \mbox{100d} & \recMeanLayer{2} \\
       &                  &                  & \recStdLayer{2} \\};
  \path[-stealth]
    (m-1-1) edge node [above] {lin+$\tanh$} (m-1-2) 
    (m-1-2) edge node [above] {lin+$\tanh$} (m-1-3)
    (m-1-3) edge node [above] {lin}         (m-1-4)
	        edge node [below] {lin+$\exp$}  (m-2-4);
\end{tikzpicture} 
\begin{tikzpicture}
 \node at (0, 0.3) {$\pmfGen(\hidLayer{1}|\hidLayer{2})=\normal(\hidLayer{1}|\genMeanLayer{1},\mathrm{diag}(\genStdLayer{1}))$} ;
  \matrix (m) [matrix of math nodes,row sep=0em,column sep=4em,minimum width=2em,nodes={draw}] at (7, 0)
  {
     \hidLayer{2} & \mbox{100d} & \mbox{100d} & \genMeanLayer{1} \\
       &                  &                  & \genStdLayer{1} \\};
  \path[-stealth]
    (m-1-1) edge node [above] {lin+$\tanh$} (m-1-2) 
    (m-1-2) edge node [above] {lin+$\tanh$} (m-1-3)
    (m-1-3) edge node [above] {lin}         (m-1-4)
	        edge node [below] {lin+$\exp$}  (m-2-4);
\end{tikzpicture} 
\begin{tikzpicture}
 \node at (0, 0) {$\pmfGen(\obs|\hidLayer{1})=\mathrm{Bernoulli}(\obs|\genMeanLayer{0})$} ;
  \matrix (m) [matrix of math nodes,row sep=0em,column sep=4em,minimum width=2em,nodes={draw}] at (7.25, 0)
  {
     \hidLayer{1} & \mbox{200d} & \mbox{200d} & \genMeanLayer{0} \\};
  \path[-stealth]
    (m-1-1) edge node [above] {lin+$\tanh$} (m-1-2) 
    (m-1-2) edge node [above] {lin+$\tanh$} (m-1-3)
    (m-1-3) edge node [above] {lin+sigm}         (m-1-4);
\end{tikzpicture} 


\subsection*{Distribution of activity statistic}

In Section~\ref{sec:latent}, we defined the activity statistic $A_\latentUnit = \mathrm{Cov}_{\obs} \left( \expect_{\latentUnit \sim \pmfRec(\latentUnit | \obs)} [\latentUnit] \right)$, and chose a threshold of $10^{-2}$ for determining if a unit is active. One justification for this is that the distribution of this statistic consisted of two widely separated modes in every case we looked at. Here is the histogram of $\log A_\latentUnit$ for a VAE with one stochastic layer:
\begin{center}
\includegraphics[width=0.3\textwidth]{\figuresdir{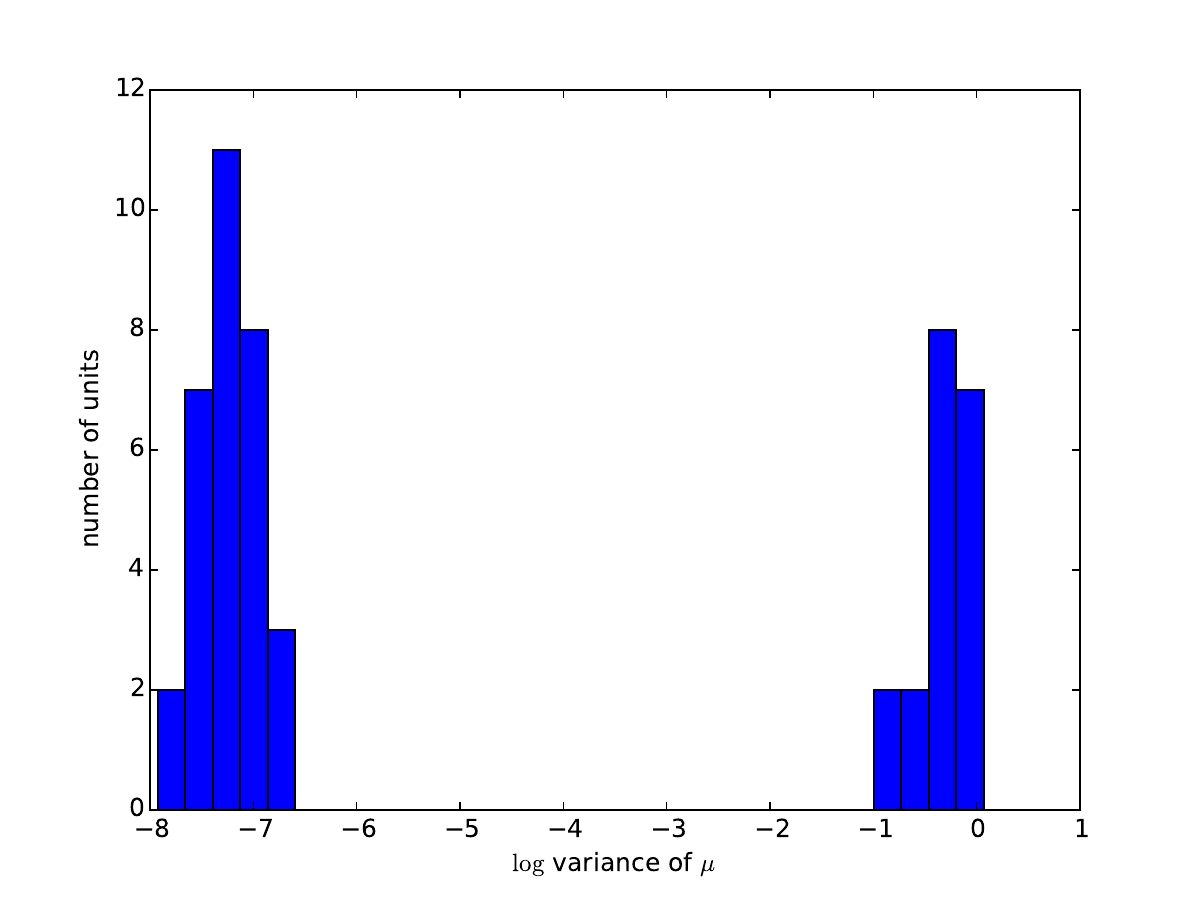}}
\end{center}

\subsection*{Visualization of posterior distributions}

We show some examples of true and approximate posteriors for VAE and IWAE models trained with two latent dimensions. Heat maps show true posterior distributions for 6 training examples, and the pictures in the bottom row show the examples and their reconstruction from samples from $\pmfRec(\hid | \obs)$.  {\bf Left:} VAE. {\bf Middle:} IWAE, with $\numSamples = 5$. {\bf Right:} IWAE, with $\numSamples = 50$. The IWAE prefers less regular posteriors and more spread out posterior predictions.

\includegraphics[width=0.2\textwidth]{\figuresdir{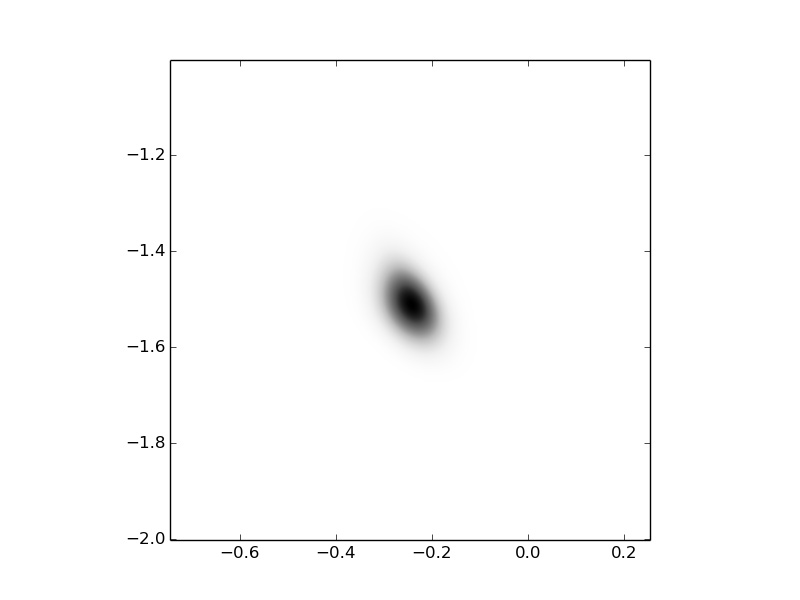}} \hspace{2ex} \includegraphics[width=0.2\textwidth]{\figuresdir{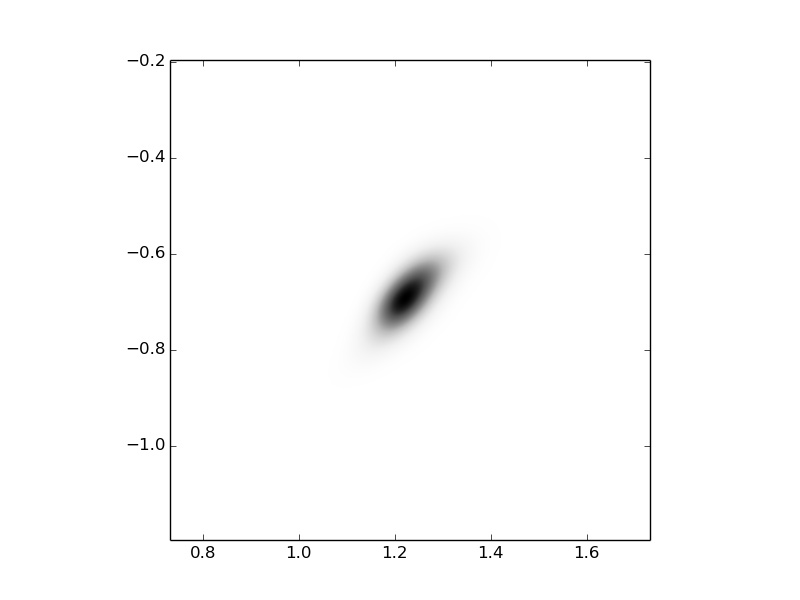}} \hspace{2ex}\includegraphics[width=0.2\textwidth]{\figuresdir{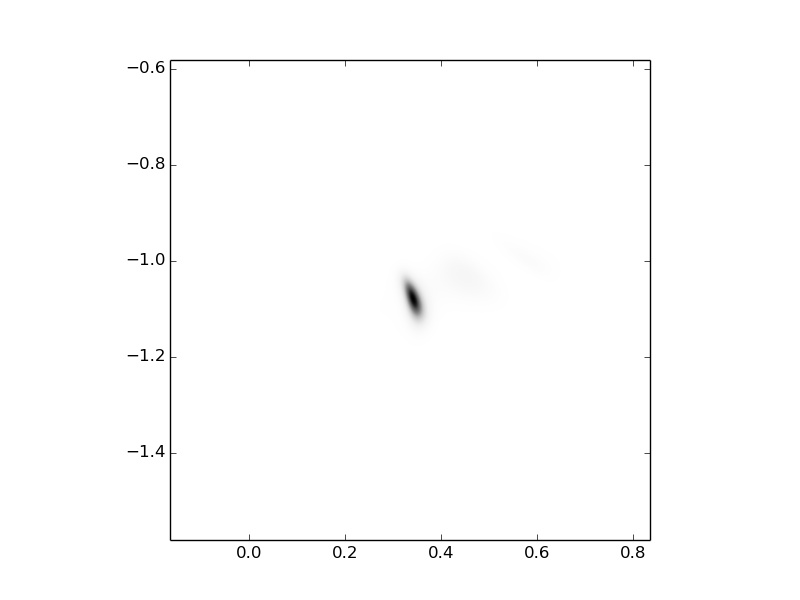}} \\
\includegraphics[width=0.2\textwidth]{\figuresdir{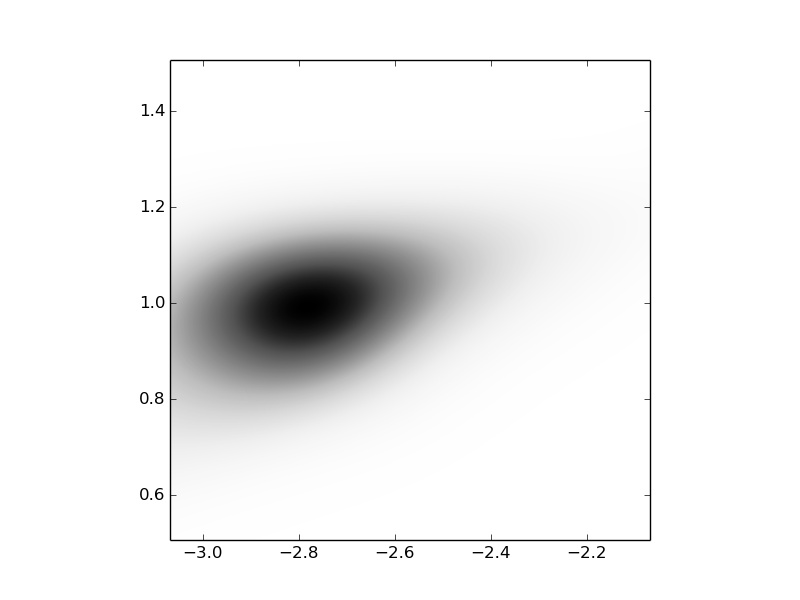}} \hspace{2ex} \includegraphics[width=0.2\textwidth]{\figuresdir{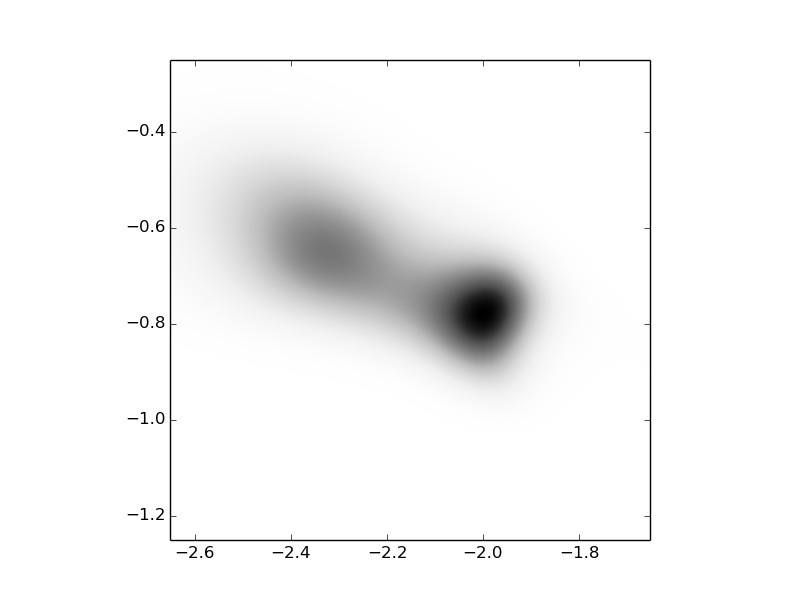}} \hspace{2ex} \includegraphics[width=0.2\textwidth]{\figuresdir{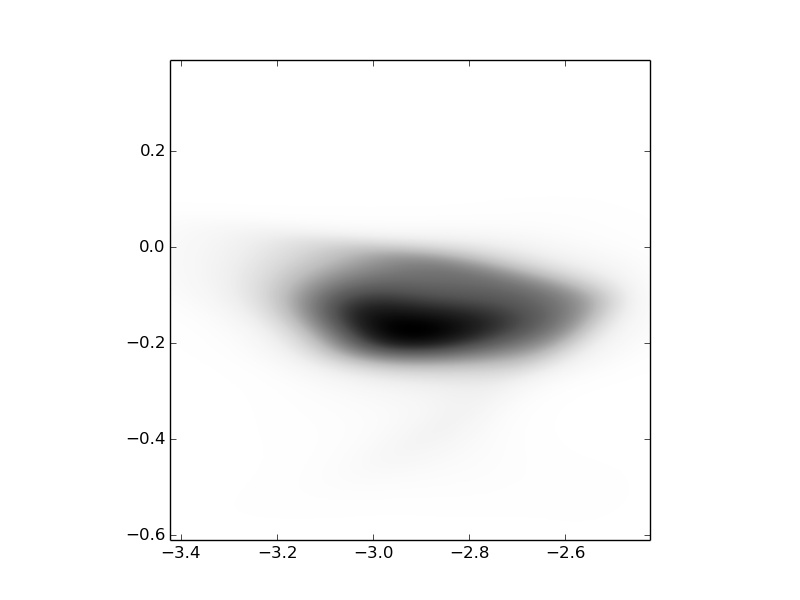}} \\
\includegraphics[width=0.2\textwidth]{\figuresdir{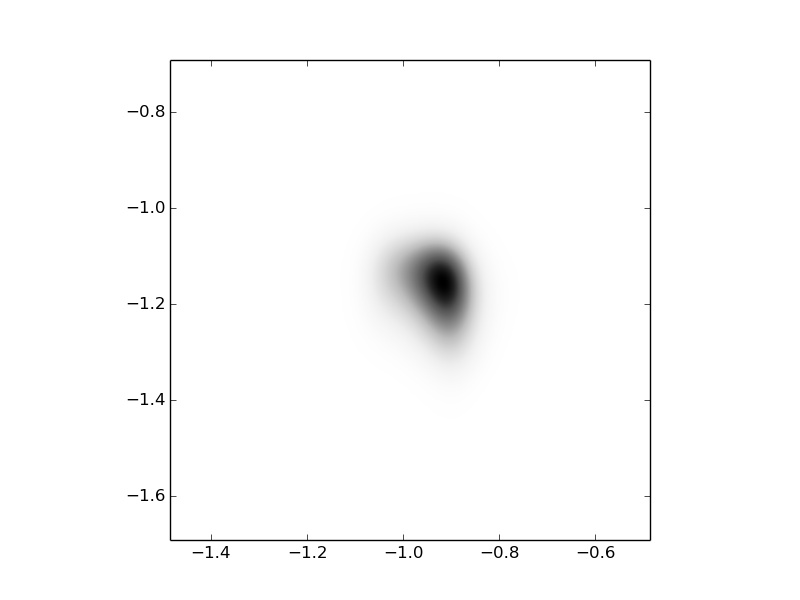}} \hspace{2ex} \includegraphics[width=0.2\textwidth]{\figuresdir{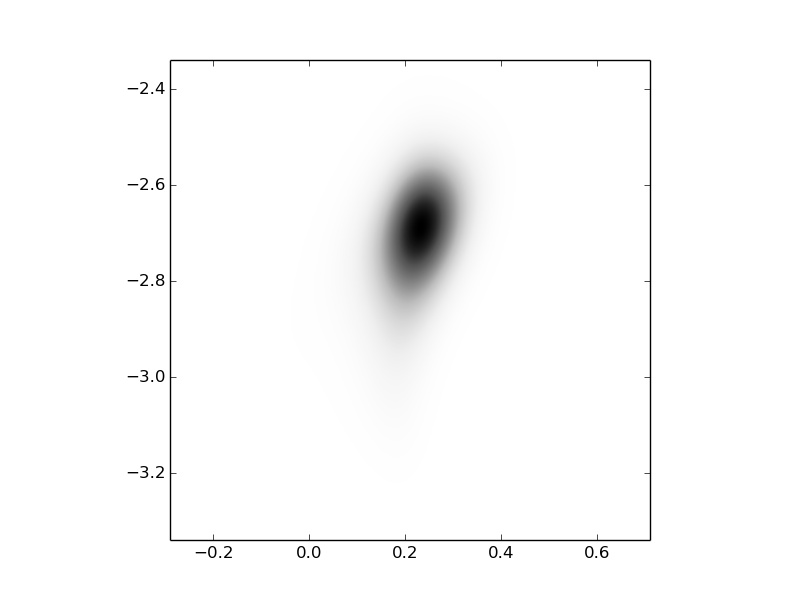}} \hspace{2ex} \includegraphics[width=0.2\textwidth]{\figuresdir{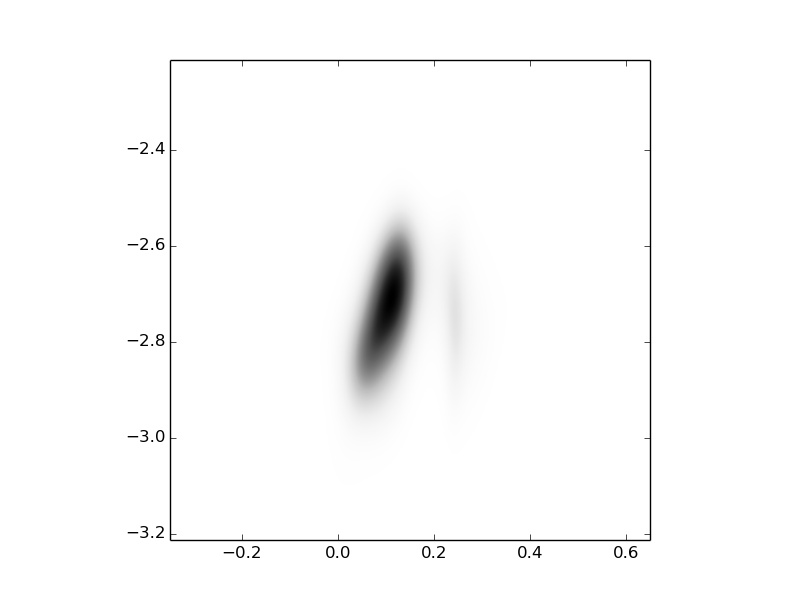}} \\
\includegraphics[width=0.2\textwidth]{\figuresdir{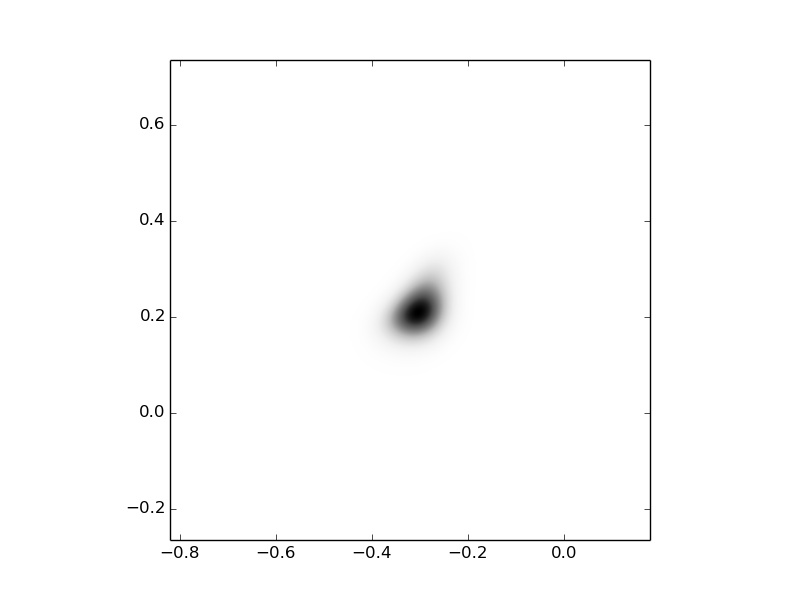}} \hspace{2ex} \includegraphics[width=0.2\textwidth]{\figuresdir{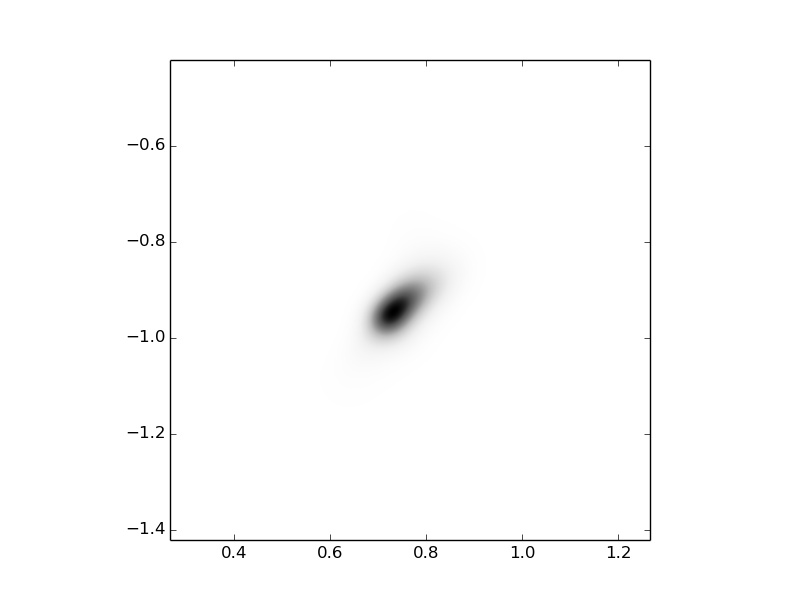}} \hspace{2ex} \includegraphics[width=0.2\textwidth]{\figuresdir{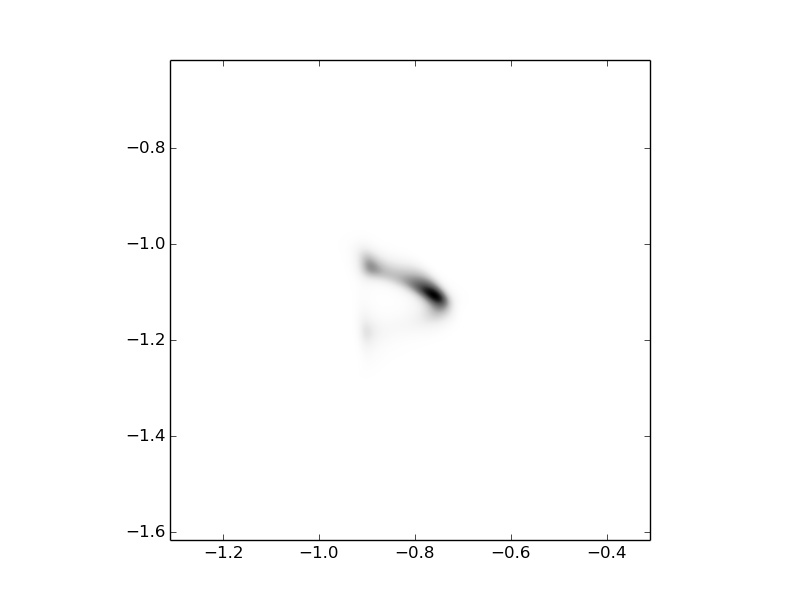}} \\
\includegraphics[width=0.2\textwidth]{\figuresdir{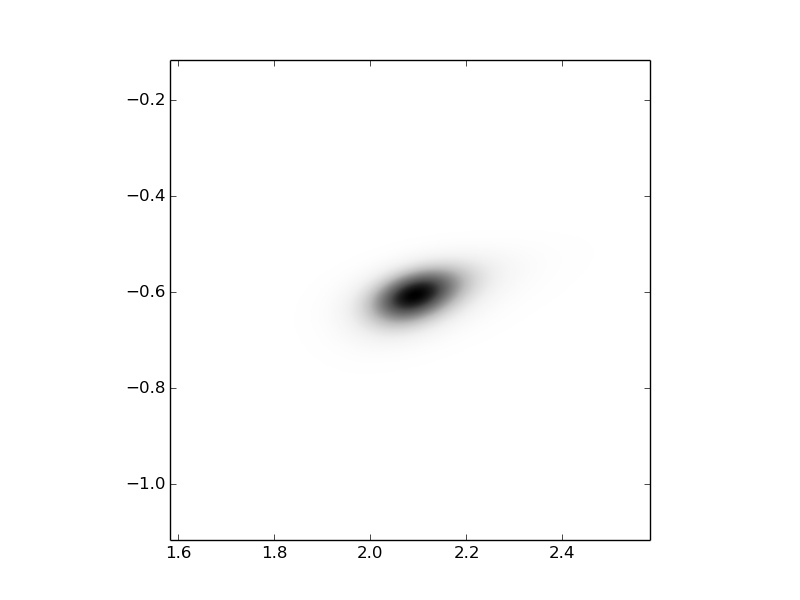}} \hspace{2ex} \includegraphics[width=0.2\textwidth]{\figuresdir{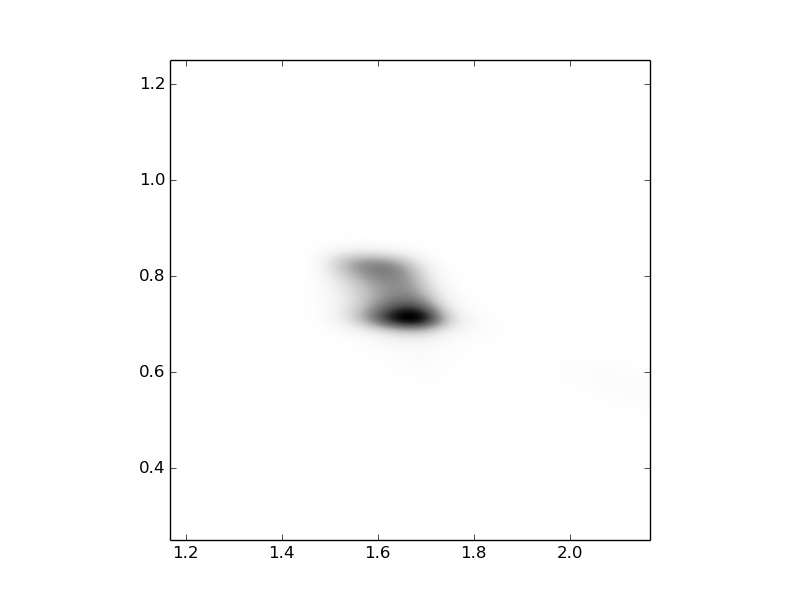}} \hspace{2ex} \includegraphics[width=0.2\textwidth]{\figuresdir{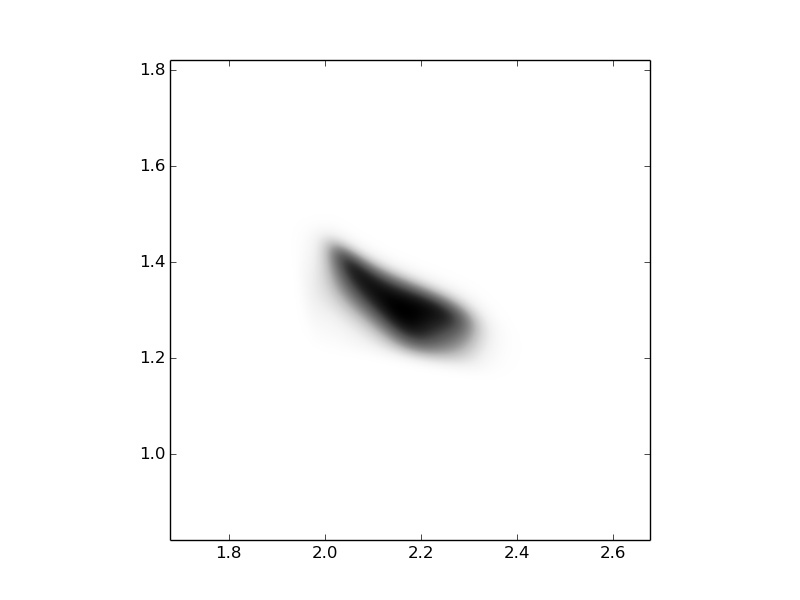}} \\
\includegraphics[width=0.2\textwidth]{\figuresdir{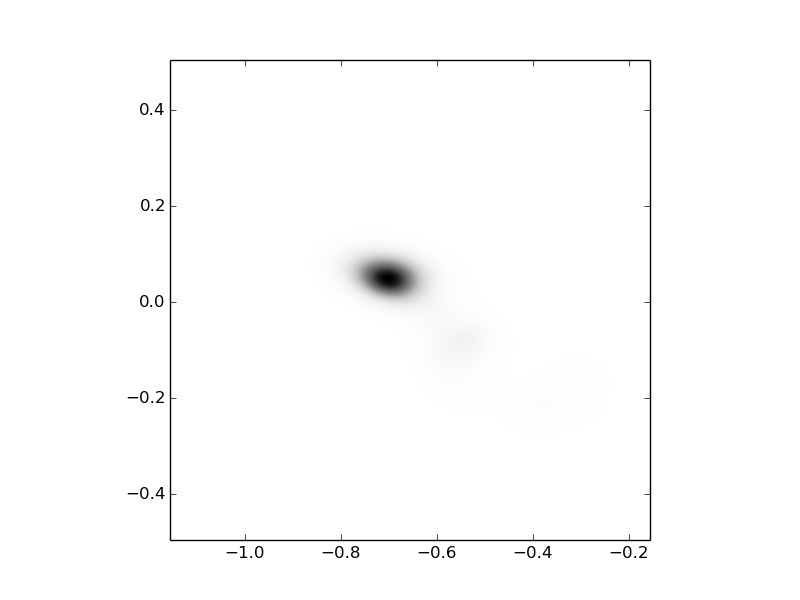}} \hspace{2ex} \includegraphics[width=0.2\textwidth]{\figuresdir{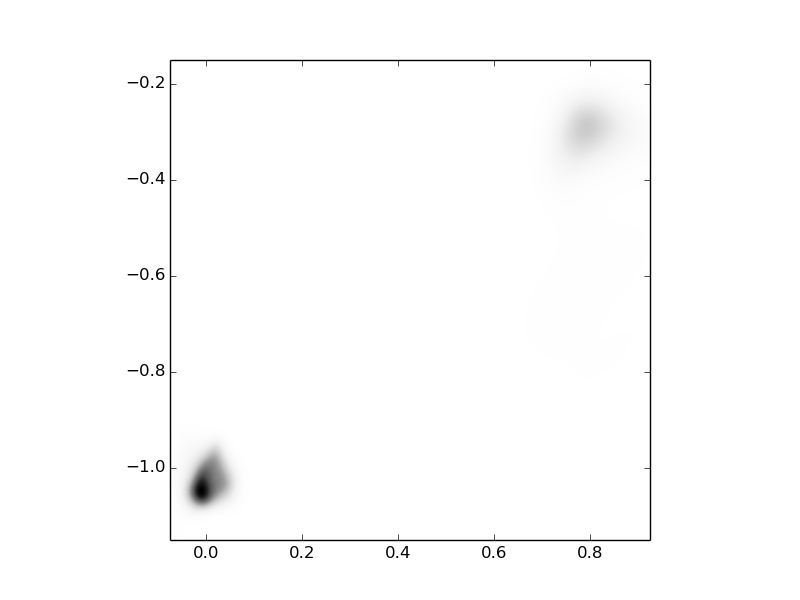}} \hspace{2ex} \includegraphics[width=0.2\textwidth]{\figuresdir{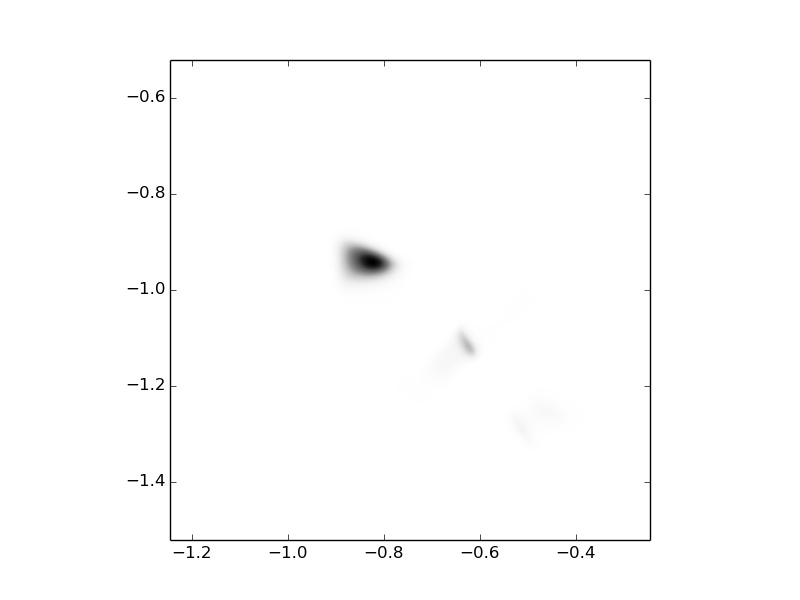}} \\
\includegraphics[width=0.2\textwidth]{\figuresdir{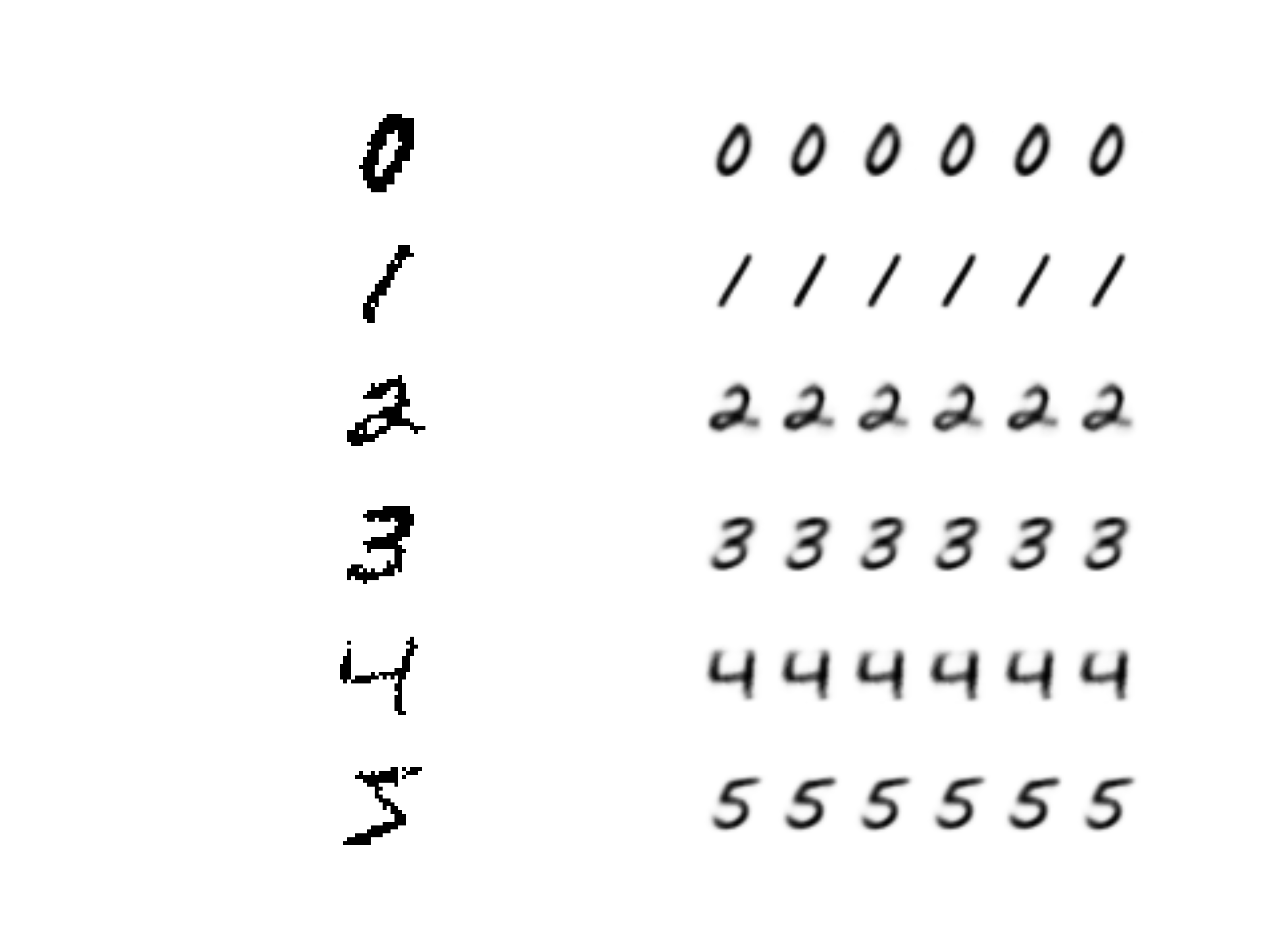}} \hspace{2ex} \includegraphics[width=0.2\textwidth]{\figuresdir{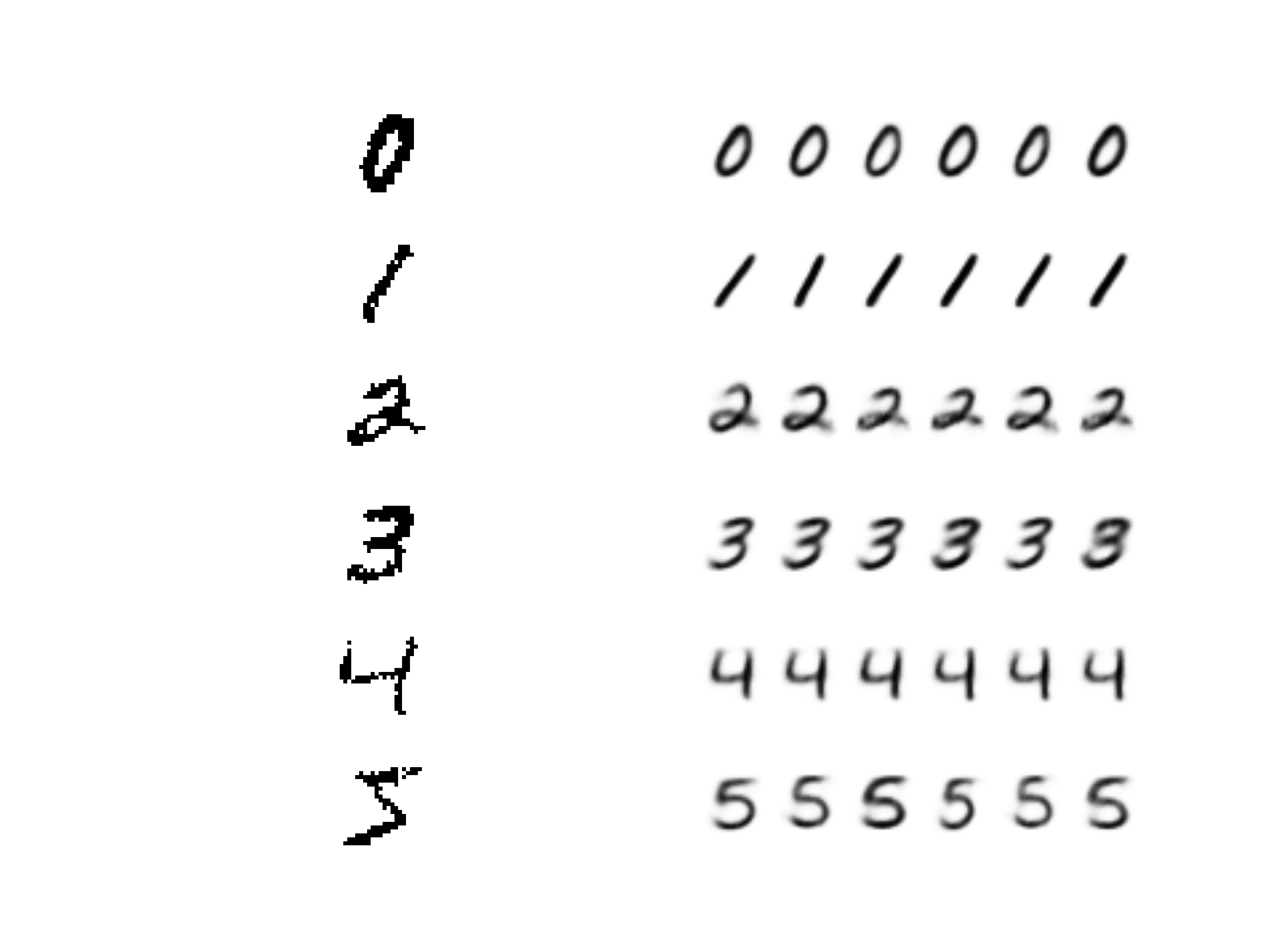}} \hspace{2ex} \includegraphics[width=0.2\textwidth]{\figuresdir{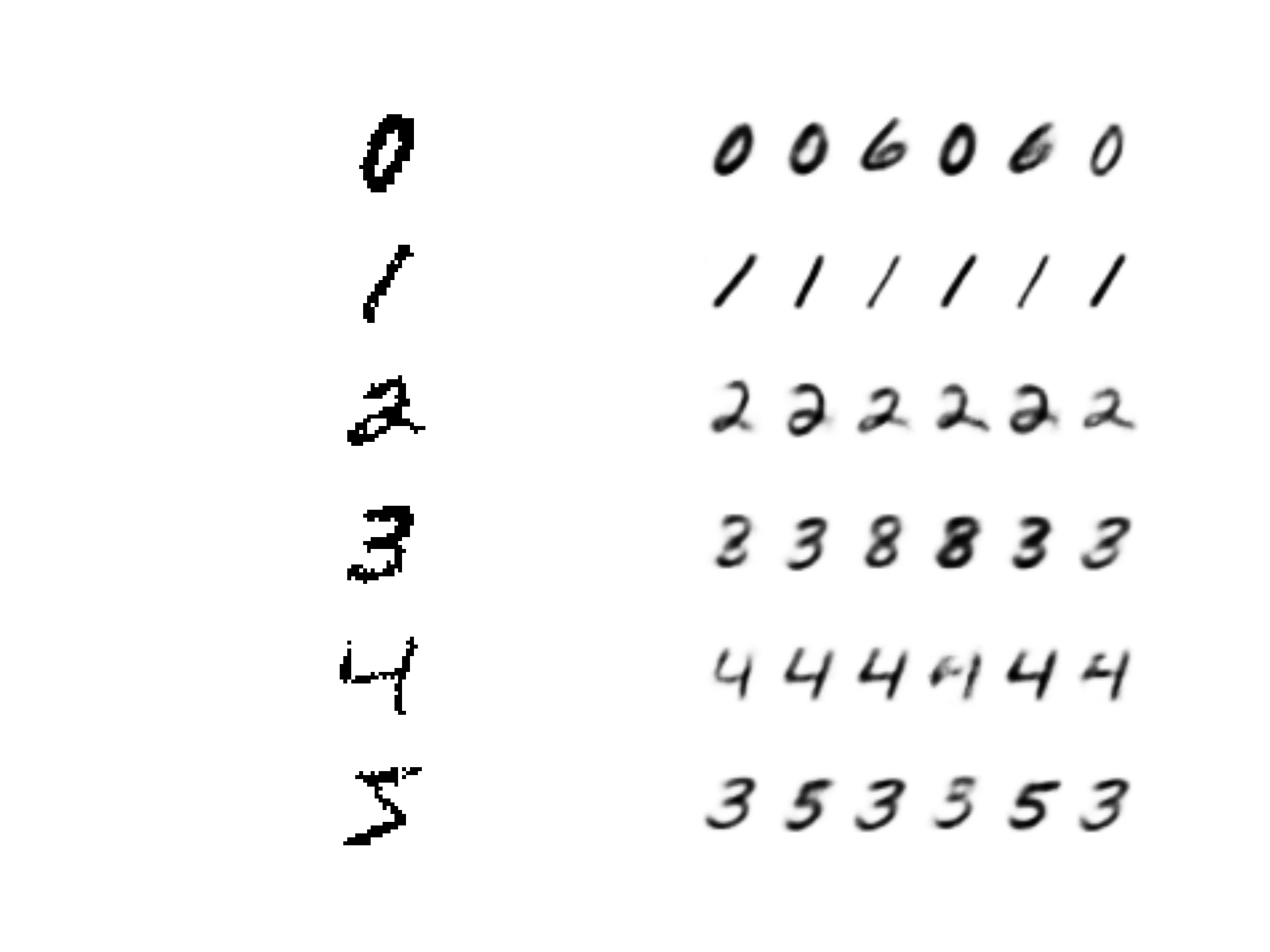}}

\section*{Appendix D}
\label{app:binarization}

\subsection*{Results for a fixed MNIST binarization}

Several previous works have used a fixed binarization of the MNIST dataset defined by \cite{binarization}. We repeated our experiments training the models on the 50000 examples from the training dataset, and evaluating them on the 10000 examples from the test dataset. Otherwise we used the same training procedure and hyperparameters as in the experiments in the main part of the paper. The results in table \ref{tbl:results_bin_fix} indicate that the conclusions about the relative merits of VAEs and IWAEs are unchanged in the new experimental setup. In this setup we noticed significantly larger amounts of overfitting.

\vspace{2ex}

\begin{table}[htbp]
\vspace{-0.3in}
\begin{center}
\begin{tabular}{@{}ll cc cc cc cc cc cc@{}}

              &     & \multicolumn{2}{c}{VAE} & \multicolumn{2}{c}{IWAE} \\
\cmidrule(r){3-4} \cmidrule(l){5-6}  
\shortstack{\# stoch. \\ layers} & $\numSamples$ & NLL               & \shortstack{active \\ units}             & NLL               & \shortstack{active \\ units}   \\ 
\cmidrule(lr){1-1} \cmidrule(lr){2-2} \cmidrule(lr){3-3}\cmidrule(lr){4-4}\cmidrule(lr){5-5}\cmidrule(lr){6-6} \cmidrule(lr){7-7}      \midrule
1             & 1   & 88.71 &  19  & 88.71 &  19   \\
              & 5   & 88.83 &  19  & 87.63 &  22   \\
              & 50  & 89.05 &  20  & 87.10 &  24   \\ \midrule
2             & 1   & 88.08 &  16+5 & 88.08 &  16+5 \\
              & 5   & 87.63 &  17+5 & 86.17&  21+5 \\
              & 50  & 87.86 &  17+6 & 85.32&  24+7 \\ \bottomrule
\end{tabular}
\end{center}
\vspace{-0.1in}
\caption{\small Results on density estimation and the number of active latent dimensions on the fixed binarization MNIST dataset. For models with two latent layers, ``$k_1+k_2$'' denotes $k_1$ active units in the first layer and $k_2$ in the second layer. The generative performance of IWAEs improved with increasing $\numSamples$, while that of VAEs benefitted only slightly. Two-layer models achieved better generative performance than one-layer models. }
\label{tbl:results_bin_fix}
\end{table}

\section*{Appendix E}
\label{app:samples}

\subsection*{Samples}
\vspace{2ex}
\begin{table}[htbp]
\vspace{-0.3in}
\begin{center}
\begin{tabular}[t]{@{}ccc@{}}
 \includegraphics[width=0.25\textwidth]{\figuresdir{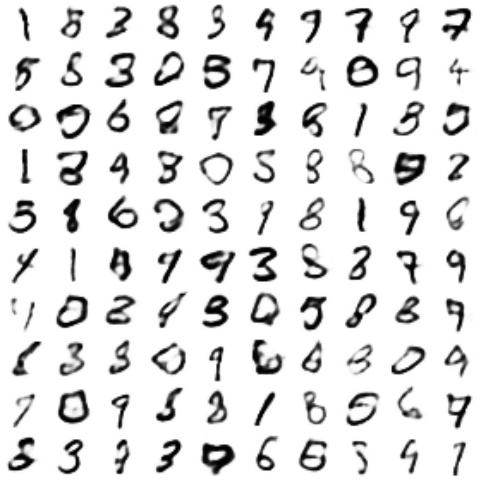}} & \hspace{3ex} &
\includegraphics[width=0.25\textwidth]{\figuresdir{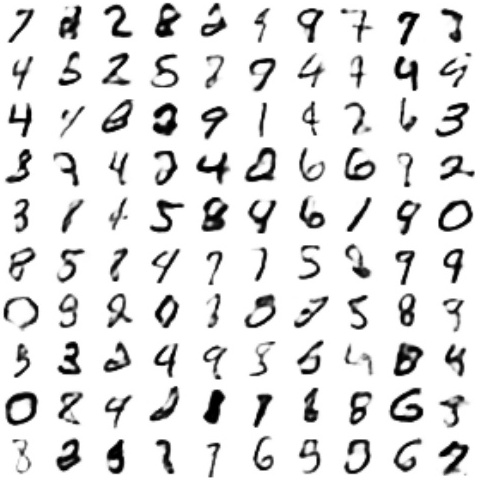}} \\
& & \\
 \includegraphics[width=0.25\textwidth]{\figuresdir{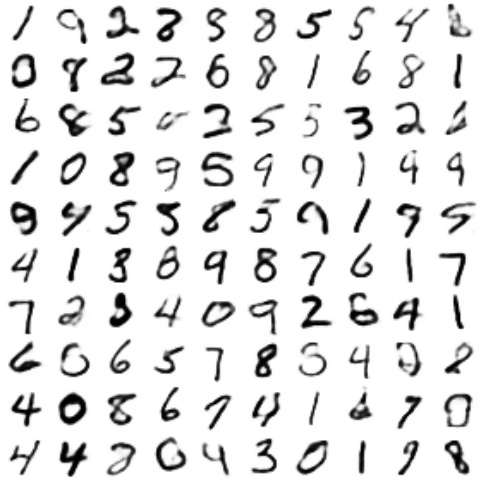}} & \hspace{3ex} &
\includegraphics[width=0.25\textwidth]{\figuresdir{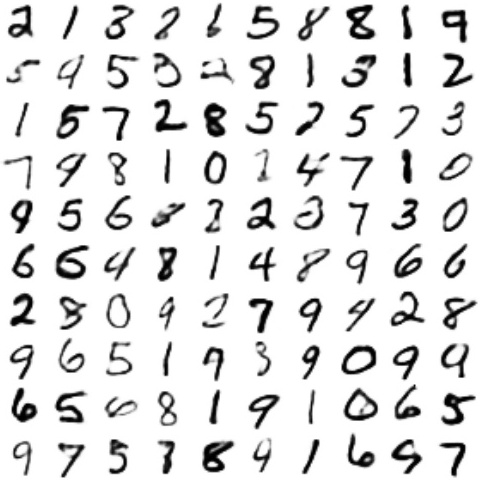}} \\
\end{tabular}
\end{center}
\vspace{-0.1in}
\caption{\small Random samples from VAE (left column) and IWAE with $k=50$ (right column) models. Row 1: models with one stochastic layer. Row 2: models with two stochastic layers. Samples are represented as the means of the corresponding Bernoulli distributions.}
\label{tbl:samples}
\end{table}

\end{document}